\documentclass{article}

\usepackage{iclr2027_conference,times}
\usepackage{amsmath,amssymb,amsthm}
\usepackage{algorithm}
\usepackage{algpseudocode}
\usepackage{microtype}
\usepackage[hidelinks]{hyperref}
\usepackage{url}

\newcommand{\Pcal}{\mathcal P}
\newcommand{\eps}{\varepsilon}
\newcommand{\Gap}{\mathsf G}
\newcommand{\TV}{\operatorname{TV}}

\newtheorem{theorem}{Theorem}[section]
\newtheorem{lemma}[theorem]{Lemma}

\newtheorem{corollary}[theorem]{Corollary}

\title{Hypothesis Testing with Conditional Queries:\\Learnability and the Value of Interaction}

\author{Zonghuan Xu\\
Fudan University}

\iclrfinalcopy
\hypersetup{pdftitle={Hypothesis Testing with Conditional Queries: Learnability and the Value of Interaction},pdfauthor={Zonghuan Xu}}
\begin{document}

\maketitle
\lhead{Preprint}

\begin{abstract}
Model evaluations may fix all tests before observing any responses or select
later tests using earlier responses. We study this choice in a
conditional-query model on a finite outcome space $\mathcal X$ with
$|\mathcal X|=N$. We first ask which pairs
of distribution classes can be reliably distinguished. We then ask how many
additional queries are required to match an adaptive tester when all queried
events must be fixed in advance. We show that learnability holds if and only if
the two classes have positive separation in their pairwise conditional
probabilities. When this separation is zero, the optimal worst-case error is
exactly $1/2$ at every finite query budget. For any $T$-query adaptive policy
and any $\rho\in(0,1)$, we construct a randomized non-adaptive procedure using
$O\bigl(N^2(T+\log(1/\rho))\bigr)$ pair queries chosen before any response is
observed. Its simulated transcript is within $\rho$ in total variation of the
adaptive transcript, uniformly over all distributions in the model. We also
construct a matching family with constant adaptive query complexity and
$\Omega_{\eps}(N^2)$ non-adaptive query complexity. Consequently, the
worst-case fixed-error adaptivity gap is $\Theta_{\eps}(N^2)$. Thus interaction
can reduce the required number of tests by a quadratic factor, but the apparent
exponential branching of an interactive evaluation does not yield an
exponential query advantage.
\end{abstract}

\section{Introduction}

A basic choice in model evaluation is whether to fix all tests before observing any responses or to choose subsequent tests based on earlier responses. Widely used benchmark suites typically fix their items before a model is evaluated \citep{hendrycks2021,srivastava2023,liang2023}. We focus on the consequences of this choice within a single evaluation. A static design avoids sequential rounds, but it may need to prepare tests for many possible response histories. An interactive evaluation can wait to see which history occurs before selecting the next test. Interaction may therefore reduce the number of tests at the cost of sequential evaluation. How large can this reduction be? In particular, can a fixed test suite match any interactive evaluation with a polynomial, rather than exponential, increase in size?

In real interactive evaluation, evaluators may adjust the range of behaviors examined based on earlier results, while these adjustments may also change the model's response tendencies. Changes in the evaluation strategy and in model responses are thus coupled. To isolate the former mechanism, we adopt an idealized model: evaluation restricts the range of behaviors observed in a query without changing their relative probabilities within that range. The resulting response follows the original distribution conditioned on the selected set.

We study this conditional-query model on a finite outcome space $\mathcal X=[N]$. The $N$ outcomes are task-relevant behavior categories, and a category may include multiple concrete outputs. The unknown object is a full-support distribution $P$ on $\mathcal X$. A binary testing target specifies two hypothesis classes, $\Pcal_0$ and $\Pcal_1$, with the promise that $P\in\Pcal_0\cup\Pcal_1$. The tester must determine which class contains $P$. We compare adaptive and non-adaptive testers under the same conditional-query model. Both may select any nonempty event $A\subseteq\mathcal X$ and receive a sample from $P(\cdot\mid A)$. The only difference is when the queried events are chosen. A non-adaptive tester fixes all events before observing any samples, whereas an adaptive tester may select each event from the observed history.

The model raises two questions. First, which pairs $(\Pcal_0,\Pcal_1)$ are learnable at all, in the sense that their worst-case error can be driven arbitrarily close to zero as the query budget grows? Second, among learnable pairs, how many additional queries are needed when all queried events must be fixed in advance? Since $\mathcal X$ has $2^N-1$ nonempty events, a non-adaptive tester could sample every event sufficiently many times and then replay any finite-round adaptive policy. This elementary argument gives an exponential upper bound, but does not reveal the true cost of removing adaptive rounds. We seek both an exact learnability criterion and the smallest worst-case increase in query count. We call the ratio between the optimal non-adaptive and adaptive query complexities at a common error level the \emph{adaptivity gap}.

Our conditional-query model connects most directly to conditional distribution testing, which studies algorithms that sample from a distribution conditioned on a chosen event \citep{chakraborty2013,canonne2015,canonne2014}. Existing results establish problem-specific upper and lower bounds for identity, equivalence, uniformity, support-size, and tolerant testing \citep{acharya2018,falahatgar2015,narayanan2021}. The area is surveyed by \citet{canonne2020}. Most closely related to our efficiency question, \citet{kamath2019} give non-adaptive query upper bounds of $\widetilde O(\log^2 N)$ and $\widetilde O(\log N)$ for identity and uniformity testing, respectively, at fixed distance threshold and error rate. These tasks therefore have at most polylogarithmic adaptivity gaps. \citet{chakraborty2024} study bounded adaptivity for equivalence testing. \citet{fotakis2022} provide tools for exact sampling from pairwise comparisons. Our work considers arbitrary pairs of hypothesis classes of full-support distributions, characterizing uniform learnability and determining the worst-case query-efficiency gap through universal transcript simulation and a matching lower bound. Additional background on adaptive evaluation and active hypothesis testing appears in the appendix.

Our contributions are summarized below.
\begin{itemize}

\item \textbf{An exact characterization of conditional-query learnability.}
Define the pairwise conditional map $\Theta(P)=(\theta_{ij}(P))_{i<j}$, where $\theta_{ij}(P)=P(i\mid\{i,j\})$. A pair of hypothesis classes is learnable if and only if its two images under $\Theta$ have positive pairwise separation. When this separation is zero, the optimal worst-case error under either adaptive or non-adaptive testing is exactly $1/2$ for every finite query budget.

\item \textbf{A universal non-adaptive simulation bound.}
For any $T$-round adaptive policy and any $\rho\in(0,1)$, we construct a randomized non-adaptive procedure using $O\bigl(N^2(T+\log(1/\rho))\bigr)$ pair queries whose simulated transcript is within $\rho$ in total variation of the adaptive transcript, uniformly over all full-support distributions. The entire pair sequence is chosen before any response is observed, and the bound does not depend on the cardinality or pairwise separation of the hypothesis classes.

\item \textbf{Tightness of the non-adaptive simulation bound.}
We construct a family whose adaptive query complexity is independent of $N$ at fixed error, while every non-adaptive test requires $\Omega_{\eps}(N^2)$ queries. This example shows that the quadratic dependence on $N$ in the preceding simulation bound is unavoidable and gives the exact worst-case rate
\[
\Gap_N(\eps)=\Theta_{\eps}(N^2).
\]

\end{itemize}

Thus, even though an adaptive evaluation may branch into exponentially many response histories, its worst-case query advantage in this model is quadratic in $N$ at fixed error.

\section{Problem Setup}
\label{sec:setup}

We model an evaluated system by an unknown distribution $P$ over a finite set
of observable outcomes. The evaluation must distinguish between two
alternatives, $P\in\Pcal_0$ and $P\in\Pcal_1$, using samples from conditional
distributions selected by the tester. Choosing every conditioning event before
observing any responses gives the non-adaptive protocol, while allowing later
events to depend on earlier observations gives the adaptive protocol. This
section defines both protocols and the quantities used to compare them.

\paragraph{Hypotheses and conditional queries.}
Let $\mathcal X$ be a finite outcome space with $|\mathcal X|=N$, and let
\[
\Delta_+(\mathcal X)
:=
\left\{P\in\Delta(\mathcal X):P(x)>0\text{ for every }x\in\mathcal X\right\}.
\]
The testing problem is specified by two nonempty hypothesis classes
$\Pcal_0,\Pcal_1\subseteq\Delta_+(\mathcal X)$. These classes may be arbitrary
subsets of $\Delta_+(\mathcal X)$ and need not be finite. Under a hidden label
$b\in\{0,1\}$, the unknown distribution $P$ is selected from $\Pcal_b$. The
tester observes $P$ only through conditional queries and must estimate $b$.

The query family is the collection of nonempty events
$\mathcal A:=2^{\mathcal X}\setminus\{\varnothing\}$. A query $A\in\mathcal A$
returns a fresh observation from
\[
P_A(x)
:=
P(x\mid A)
=
\frac{P(x)\mathbf 1\{x\in A\}}{P(A)}.
\]
Full support ensures that every nonempty event is a valid conditional query.

\paragraph{Adaptive and non-adaptive testers.}
An adaptive tester may select each event after observing the responses to its
earlier queries. After $t-1$ queries, the complete observed history is
\[
H_{t-1}:=(A_1,X_1,\ldots,A_{t-1},X_{t-1}).
\]
A randomized adaptive policy selects
\[
A_t\sim\pi_t(\cdot\mid H_{t-1}),
\qquad
X_t\sim P_{A_t},
\]
where $\pi_t(\cdot\mid H_{t-1})$ is a probability distribution on
$\mathcal A$. After $T$ queries, a possibly randomized decision rule outputs
$\widehat b\in\{0,1\}$. A non-adaptive tester instead draws the entire query
vector $A_{1:T}$ before observing any samples. The queries may be randomized,
correlated, and repeated, but their joint distribution cannot depend on the
observations. For a fixed testing policy, we write $\mathsf L_P(H_T)$ for the probability law of the transcript $H_T$ when the underlying distribution is $P$, including randomness from both the policy and the oracle responses.

\paragraph{Risk, learnability, and the adaptivity gap.}
The minimax risk records the smallest worst-case error achievable with a fixed
query budget. The adaptive minimax risk at horizon $T$ is
\[
R_T^{\mathrm{ad}}(\Pcal_0,\Pcal_1)
:=
\inf_{\text{adaptive testers}}
\max_{b\in\{0,1\}}
\sup_{P\in\Pcal_b}
\Pr_{P}(\widehat b\neq b).
\]
The non-adaptive risk $R_T^{\mathrm{na}}(\Pcal_0,\Pcal_1)$ is defined by
restricting the infimum to non-adaptive testers. Therefore
\[
R_T^{\mathrm{ad}}(\Pcal_0,\Pcal_1)
\le
R_T^{\mathrm{na}}(\Pcal_0,\Pcal_1)
\le \frac12.
\]

Learnability asks whether this risk can be made arbitrarily small as the query
budget grows. We call $(\Pcal_0,\Pcal_1)$ adaptively conditionally learnable if
$R_T^{\mathrm{ad}}(\Pcal_0,\Pcal_1)\to0$ as $T\to\infty$. Non-adaptive
conditional learnability is defined in the same way using
$R_T^{\mathrm{na}}$. For a target error $\eps\in(0,1/2)$, define
\[
T_{\mathrm{ad}}^\star(\eps;\Pcal_0,\Pcal_1)
:=
\inf\left\{T\ge 0:R_T^{\mathrm{ad}}(\Pcal_0,\Pcal_1)\le\eps\right\},
\]
and define $T_{\mathrm{na}}^\star(\eps;\Pcal_0,\Pcal_1)$ analogously. The
infimum of an empty set is $\infty$.

For a class pair with finite adaptive query complexity at error $\eps$, its
fixed-error adaptivity gap is
\[
\operatorname{Gap}_{\eps}(\Pcal_0,\Pcal_1)
:=
\frac{T_{\mathrm{na}}^\star(\eps;\Pcal_0,\Pcal_1)}
{T_{\mathrm{ad}}^\star(\eps;\Pcal_0,\Pcal_1)}.
\]
For a fixed outcome-space size $N$, the worst-case adaptivity gap is
\[
\Gap_N(\eps)
:=
\sup_{\substack{\Pcal_0,\Pcal_1\subseteq\Delta_+(\mathcal X)\\
1\le T_{\mathrm{ad}}^\star(\eps;\Pcal_0,\Pcal_1)<\infty}}
\operatorname{Gap}_{\eps}(\Pcal_0,\Pcal_1).
\]
We count conditional queries and place no restriction on the computation used
to choose the queries or form the final decision.

Subscripts on asymptotic notation indicate permitted dependence of the implicit constants. In particular, the constants in $O_{\eps}$, $\Omega_{\eps}$, and $\Theta_{\eps}$ may depend on $\eps$, but not on $N$.

\section{Exact Characterization of Conditional-Query Learnability}
\label{sec:learnability}

This section gives an exact criterion for conditional-query learnability. The
response law to a query on $\{i,j\}$ depends on $P$ only through the relative
weight $P(i)/P(j)$, and the collection of all such ratios determines $P$. We show that two hypothesis classes are
learnable precisely when their images under this pairwise representation have
a positive uniform separation. Sufficiency follows by estimating every pair in
advance. Necessity follows by showing that pairwise closeness controls every
conditional query and therefore every finite adaptive transcript.

\paragraph{Pairwise representation and main theorem.}
Fix an arbitrary labeling $\mathcal X=\{1,\ldots,N\}$ and let
$M:=\binom N2$. For $1\le i<j\le N$, define
\[
\theta_{ij}(P)
:=
P(i\mid\{i,j\})
=
\frac{P(i)}{P(i)+P(j)},
\qquad
\Theta(P)
:=
\bigl(\theta_{ij}(P)\bigr)_{i<j}\in(0,1)^M.
\]
For $i<j$, these coordinates determine the ratio
\[
\frac{P(i)}{P(j)}
=
\frac{\theta_{ij}(P)}{1-\theta_{ij}(P)}.
\]
Together with the normalization
$\sum_iP(i)=1$, these ratios determine $P$.
The same pairwise normalization underlies classical probabilistic models for
paired comparisons \citep{bradley1952,hunter2004}.
For $P,Q\in\Delta_+(\mathcal X)$, write
\[
d_{\mathrm{pair}}(P,Q)
:=
\|\Theta(P)-\Theta(Q)\|_\infty.
\]
The pairwise separation between the two hypothesis classes is
\[
\Delta_{\mathrm{pair}}(\Pcal_0,\Pcal_1)
:=
\inf_{P\in\Pcal_0,\,Q\in\Pcal_1}
d_{\mathrm{pair}}(P,Q).
\]

\begin{theorem}[Exact learnability criterion]
\label{thm:learnability}
Suppose $N\ge2$. The following statements are equivalent.
\begin{enumerate}
\item $\Delta_{\mathrm{pair}}(\Pcal_0,\Pcal_1)>0$.
\item $R_T^{\mathrm{na}}(\Pcal_0,\Pcal_1)\to0$ as $T\to\infty$.
\item $R_T^{\mathrm{ad}}(\Pcal_0,\Pcal_1)\to0$ as $T\to\infty$.
\end{enumerate}
If $\Delta_{\mathrm{pair}}(\Pcal_0,\Pcal_1)=0$, then for every finite $T$,
\[
R_T^{\mathrm{ad}}(\Pcal_0,\Pcal_1)
=
R_T^{\mathrm{na}}(\Pcal_0,\Pcal_1)
=
\frac12.
\]
If $\Delta:=\Delta_{\mathrm{pair}}(\Pcal_0,\Pcal_1)>0$, then for every
$\eps\in(0,1/2)$,
\[
T_{\mathrm{na}}^\star(\eps;\Pcal_0,\Pcal_1)
\le
M\left\lceil
\frac{8}{\Delta^2}\log\frac{2M}{\eps}
\right\rceil.
\]
\end{theorem}

\paragraph{Sufficiency: positive separation gives a non-adaptive test.}
Suppose $\Delta:=\Delta_{\mathrm{pair}}(\Pcal_0,\Pcal_1)>0$. Query every
pair the same number of times and let $\widehat\Theta$ collect the empirical
frequencies. For $b\in\{0,1\}$, define
\[
D_b
:=
\inf_{Q\in\Pcal_b}
\|\widehat\Theta-\Theta(Q)\|_\infty,
\]
and output the label with smaller $D_b$. If the true distribution
$P\in\Pcal_b$ satisfies
$\|\widehat\Theta-\Theta(P)\|_\infty<\Delta/4$, then
$D_b<\Delta/4$. The triangle inequality gives
$D_{1-b}\ge\Delta-\Delta/4=3\Delta/4$. Hoeffding's inequality and a
union bound over the $M$ pairs give this event with probability at least
$1-\eps$ using the number of queries stated in the theorem. Thus positive
separation already yields a non-adaptive test.

\paragraph{Necessity: pairwise closeness controls every query.}
Fix $P,Q\in\Delta_+(\mathcal X)$ and let
$d:=d_{\mathrm{pair}}(P,Q)$. Pairwise closeness first bounds the response laws
under every conditional query. Applying this bound at each round then controls
the law of the entire adaptive transcript:
\[
\begin{aligned}
d_{\mathrm{pair}}(P,Q)=d
&\quad\Longrightarrow\quad
\max_{\varnothing\ne A\subseteq\mathcal X}\TV(P_A,Q_A)
\le (N-1)d,\\
&\quad\Longrightarrow\quad
\TV\bigl(\mathsf L_P(H_T),\mathsf L_Q(H_T)\bigr)
\le T(N-1)d.
\end{aligned}
\]

\noindent\emph{From pairs to arbitrary events.}
Fix a nonempty event $A$ and write $p=P_A$ and $q=Q_A$. For distinct
$i,j\in A$, conditioning once more on $\{i,j\}$ gives
\[
\frac{|p_iq_j-q_ip_j|}{(p_i+p_j)(q_i+q_j)}
=
\left|
\frac{p_i}{p_i+p_j}-\frac{q_i}{q_i+q_j}
\right|
\le d.
\]
Let $I_+:=\{i:p_i\ge q_i\}$ and $I_-:=A\setminus I_+$. Total variation
can be written as the mass transported across this cut:
\begin{align*}
\TV(p,q)
&=
\sum_{i\in I_+}\sum_{j\in I_-}(p_iq_j-q_ip_j)\\
&\le
d\sum_{\{i,j\}\subseteq A}(p_i+p_j)(q_i+q_j)\\
&=
d\left(1+(|A|-2)\sum_{i\in A}p_iq_i\right)
\le (|A|-1)d.
\end{align*}
Consequently,
\[
\max_{\varnothing\ne A\subseteq\mathcal X}
\TV(P_A,Q_A)
\le
(N-1)d_{\mathrm{pair}}(P,Q).
\tag{1}
\label{eq:event-stability}
\]

\medskip
\noindent\emph{From events to adaptive transcripts.}
Couple two executions of the same tester under $P$ and $Q$. As long as their
histories agree, use the same internal randomness, so the tester selects the
same next event in both executions. Couple the two conditional observations
optimally. Equation~\eqref{eq:event-stability} bounds the probability that
they first disagree at any round by $(N-1)d_{\mathrm{pair}}(P,Q)$. A union
bound over $T$ rounds therefore gives
\[
\TV\bigl(\mathsf L_P(H_T),\mathsf L_Q(H_T)\bigr)
\le
T(N-1)d_{\mathrm{pair}}(P,Q).
\]

\paragraph{Zero separation forces risk $1/2$.}
Suppose $\Delta_{\mathrm{pair}}(\Pcal_0,\Pcal_1)=0$ and fix any finite
$T$-query adaptive tester. The conclusion is immediate for $T=0$. For
$T\ge1$ and every $\eta>0$, there exist
$P\in\Pcal_0$ and $Q\in\Pcal_1$ with
$d_{\mathrm{pair}}(P,Q)<\eta/[T(N-1)]$. Their transcript distributions
are then less than $\eta$ apart in total variation. For any decision rule,
its error probabilities under $P$ and $Q$ sum to at least $1-\eta$, so its
worst-case error is at least $(1-\eta)/2$. Letting $\eta$ tend to zero and
comparing with a fair random guess gives the exact minimax risk $1/2$.
Appendix~\ref{app:learnability-proof} supplies the concentration details and
the full formal proof.

Thus pairwise separation exactly characterizes conditional-query learnability.
We now compare the query counts of adaptive and non-adaptive tests on the
learnable problems.

\section{Non-Adaptive Simulation of Adaptive Queries}
\label{sec:upper-bound}

Fix one round after the adaptive policy has selected an event $A$. The simulator must produce a draw from $P_A$ using only pair samples collected before $A$ was known. Each pair sample defines a random map on $A$ that preserves $P_A$. Composing these maps backward until all starting states have the same output gives an exact sample from $P_A$. We first describe this one-round construction, then bound its cost and apply it successively to all $T$ adaptive rounds. The construction builds on the pairwise coupling-from-the-past (CFTP) sampler of \citet{fotakis2022}, using the CFTP principle of \citet{propp1996} and the read-once approach of \citet{wilson2000}. Our analysis exploits uniform sampling over all pairs to bound the total sample consumption for simulating an adaptive transcript, independently of the unknown distribution. Appendix~\ref{app:upper-bound-proofs} gives a detailed comparison with the prior sampling results.

\begin{theorem}[Universal transcript simulation]
\label{thm:upper-bound}
Let $N\ge2$, let $T\ge1$, and let $\pi$ be any randomized adaptive policy
using $T$ conditional queries. Write $M=\binom N2$. For every
$\rho\in(0,1)$, there is a randomized non-adaptive procedure that makes at
most
\[
L
:=
\left\lceil
8M\left(T+\log\frac1\rho\right)
\right\rceil
\tag{2}
\label{eq:simulation-budget}
\]
pair queries and produces a $T$-round simulated transcript $\widetilde H_T$ such
that, for every $P\in\Delta_+(\mathcal X)$,
\[
\TV\bigl(\mathsf L_P(\widetilde H_T),\mathsf L_P(H_T^\pi)\bigr)
\le
\rho.
\]
Consequently, any decision rule with error at most $\eps$ under $\pi$ has
error at most $\eps+\rho$ when applied to the simulated transcript.
\end{theorem}

\paragraph{Construction.}
The procedure first collects a fixed random sequence of pair queries and their returned samples. Once an event $A$ is selected, we scan these stored samples to simulate a draw from $P_A$. The key observation is that, for $e\subseteq A$, a state distributed according to $P_A$, conditional on lying in $e$, has law $P_e$. Thus, replacing a state in $e$ by an independent stored sample $y\sim P_e$, while leaving states outside $e$ unchanged, defines a random update that preserves $P_A$. This update also maps both endpoints of $e$ to the same sample value $y$, merging the trajectories from these two states. We compose these updates until all possible initial states produce the same output, which the algorithm returns. To this end, each stored pair query and sample $(e,y)$ defines the map $\Phi_{e,y}^{A}:A\to A$ by
\[
\Phi_{e,y}^{A}(x)
:=
\begin{cases}
y, & e\subseteq A\text{ and }x\in e,\\
x, & \text{otherwise},
\end{cases}
\qquad x\in A.
\]
Pairs not contained in $A$ give the identity map. Algorithm~\ref{alg:simulation} starts with $G=\mathrm{id}_A$ and reads unused samples, updating $G\gets G\circ\Phi_{e,y}^{A}$ until $G$ is constant on $A$, then returning its common value. Composing new maps on the right ensures that, once determined, this common value cannot change when further maps are added; Step~2 proves that the output has law $P_A$. Samples are never reused; if the stream is exhausted, a fixed rule completes a valid transcript without querying $P$.

We use $A_t$ for the event selected in round $t$ of the adaptive policy, and $E_s$ for the $s$-th preselected pair query used by the simulator. Simulating one query $A_t$ may consume several pair samples.

\begin{algorithm}[H]
\caption{Simulating a $T$-round adaptive policy from pair samples chosen in advance}
\label{alg:simulation}
\begin{algorithmic}[1]
\Require policy $\pi$, stream length $L$
\Statex \textbf{Precommitted data collection}
\State Draw $E_{1:L}$ i.i.d.\ uniformly from $\binom{\mathcal X}{2}$, submit all
queries, and store $Y_s\sim P_{E_s}$.
\Statex \textbf{Offline reconstruction}
\State Set $s\gets1$ and $\widetilde H_0\gets\varnothing$.
\For{$t=1,\ldots,T$}
    \State Draw $A_t\sim\pi_t(\cdot\mid\widetilde H_{t-1})$.
    \State Set $G\gets\mathrm{id}_{A_t}$.
    \While{$G$ is not constant on $A_t$}
        \If{$s>L$}
            \State \Return a completion under $\pi$ with responses $\widetilde X_r=\min A_r$ for $r=t,\ldots,T$.
        \EndIf
        \State $G\gets G\circ\Phi_{E_s,Y_s}^{A_t}$ and $s\gets s+1$.
    \EndWhile
    \State Let $\widetilde X_t$ be the common value of $G$ on $A_t$.
    \State Set
    $\widetilde H_t\gets(\widetilde H_{t-1},A_t,\widetilde X_t)$.
\EndFor
\State \Return $\widetilde H_T$
\end{algorithmic}
\end{algorithm}

\paragraph{Step 1: one pair sample preserves the target law.}
Fix $A$ and one unused pair sample $(e,Y)$. If $e\not\subseteq A$, then
$\Phi_{e,Y}^{A}$ is the identity. If $e\subseteq A$, let $X\sim P_A$ be
independent of $Y\sim P_e$. For $z\in e$,
\[
\Pr\bigl(\Phi_{e,Y}^{A}(X)=z\bigr)
=
P_A(e)P_e(z)
=
\frac{P(e)}{P(A)}\frac{P(z)}{P(e)}
=
P_A(z).
\]
For $z\in A\,\setminus\,e$, the mass is unchanged. Thus every unused map preserves
$P_A$. Equivalently, conditional on $X\in e$, the law of $X$ is exactly the
law $P_e$ used to redraw the pair. This uses the nested-conditioning relation
$P_A(\cdot\mid e)=P_e$ for $e\subseteq A$. Ordinary prompts may change the relative probabilities of behaviors while changing the range examined, so this relation need not hold for their response distributions.

\paragraph{Step 2: coalescence produces an exact sample.}
Write $\Phi_1^A,\Phi_2^A,\ldots$ for the successive maps used for $A$ and set
\[
C_n^A
:=
\Phi_1^A\circ\Phi_2^A\circ\cdots\circ\Phi_n^A,
\qquad
\tau_A
:=
\inf\{n:C_n^A\text{ is constant on }A\}.
\]
Interpret $\Phi_1^A$ as the update from time $-1$ to the present, $\Phi_2^A$ as
the update from time $-2$ to time $-1$, and so on. Thus $C_n^A$ maps a state at
time $-n$ to the fixed present. For analysis, start with an independent
$Z_n\sim P_A$ at time $-n$. Repeated application of Step~1 gives
$C_n^A(Z_n)\sim P_A$.

If $C_n^A$ is constant, its output no longer depends on the unavailable initial
state $Z_n$. The order of composition ensures more: if $C_n^A\equiv w$, then
$C_n^A\circ\Phi_{n+1}^A\equiv w$. We are extending the starting time farther
into the past while keeping the present fixed, so an already determined present
value cannot change. Let $W_A$ be this common value when coalescence first
occurs. On $\{\tau_A\le n\}$, $C_n^A(Z_n)=W_A$. Step~3 shows that
$\tau_A<\infty$ almost surely, and therefore
$\TV(\mathsf L(W_A),P_A)\le\Pr(\tau_A>n)\to0$. Hence $W_A\sim P_A$.

\paragraph{Step 3: how many pair samples are needed.}
The right-composition order above is convenient for exactness but awkward to
count. For each fixed $n$, the maps are i.i.d., so reversing them does not
change their joint law. We may therefore run the same maps forward and track
the set of possible states. If $k$ states remain, the next map reduces this
number to $k-1$ exactly when its uniform pair has both endpoints in that set,
an event of probability
\[
p_k=\frac{\binom k2}{M}.
\]
It follows that, for $|A|\ge2$,
\[
\tau_A
\overset{d}{=}
\sum_{k=2}^{|A|}G_k,
\qquad
G_k\sim\operatorname{Geom}\left(\frac{\binom k2}{M}\right)
\text{ independently}.
\tag{3}
\label{eq:coalescence-law}
\]
Consequently,
\[
\mathbb E\tau_A<2M,
\qquad
\mathbb E\exp\left(\frac{\tau_A}{8M}\right)\le e.
\]
The final transition, from two possible states to one, already has expected
waiting time $M=\binom N2$. This explains the quadratic scale. Appendix
~\ref{app:upper-bound-proofs} gives the geometric-waiting-time and
exponential-moment calculations.

\paragraph{Step 4: from one response to the adaptive transcript.}
Let $H_t^\infty$ be the reconstruction from an infinite pair-sample stream and
let $S_t$ be the total number of samples consumed through round $t$. After
round $t-1$, the unused suffix is still i.i.d.\ and independent of the simulated
history because $S_{t-1}$ is a stopping time. Conditional on the next selected
event $A_t$, Steps~1--2 therefore give
\[
X_t^\infty\mid H_{t-1}^\infty,A_t\sim P_{A_t}.
\]
Induction gives $\mathsf L_P(H_T^\infty)=\mathsf L_P(H_T^\pi)$. The conditional
exponential-moment calculation in Appendix~\ref{app:upper-bound-proofs} also
gives
\[
\Pr\left[
S_T>8M\left(T+\log\frac1\rho\right)
\right]
\le\rho.
\]
Couple the finite and infinite reconstructions using the same policy randomness
and first $L$ pair samples. They agree on $\{S_T\le L\}$, so with the choice of
$L$ in equation~\eqref{eq:simulation-budget},
\[
\TV\bigl(\mathsf L_P(\widetilde H_T),\mathsf L_P(H_T^\pi)\bigr)
\le\Pr(S_T>L)\le\rho.
\]
This proves Theorem~\ref{thm:upper-bound}.

\begin{corollary}[Quadratic upper bound on the adaptivity gap]
\label{cor:quadratic-upper}
For every fixed $\eps\in(0,1/2)$,
$\Gap_N(\eps)=O_{\eps}(N^2)$. More strongly, for every learnable pair of hypothesis classes, the non-adaptive query complexity using only pair queries is at most $O_{\eps}(N^2)$ times the adaptive query complexity allowing arbitrary conditional queries, both at error level $\eps$. Consequently, the same adaptivity-gap bound holds when both testers are restricted to pair queries.
\end{corollary}

Repeat an $\eps$-error adaptive tester $O_\eps(1)$ times so that its majority
vote has error at most $\eps/2$, then apply Theorem~\ref{thm:upper-bound} with
$\rho=\eps/2$. This proves the corollary. The exact finite-error bound and the
case of an unattained minimax infimum appear in
Appendix~\ref{app:upper-bound-proofs}. The next section shows that the
quadratic factor cannot be improved.

\section{A Matching Quadratic Lower Bound}
\label{sec:lower-bound}

Corollary~\ref{cor:quadratic-upper} shows that an $O(N^2)$ increase in query
count is always sufficient to remove adaptive rounds. We construct a matching worst-case example. We first hide the difference between the two labels in a target pair, so that only a query exactly equal to this pair provides a constant amount of information about the label. All other queries, including queries to larger sets, provide negligible information even when accumulated over $O(N^2)$ queries. We then use responses from two fixed sets to reveal the indices of the target pair. An adaptive tester can therefore first locate the pair and then query it, completing the test with a constant number of queries at fixed error. A non-adaptive tester must fix its queries in advance and cannot use these responses to select the target pair; it must therefore prepare queries for a quadratic number of candidate pairs.

\paragraph{Construction of the hard family.}
Fix $m\ge2$, let $N=4m$, and partition the outcome space into
\[
I=\{a_1,\ldots,a_m\},
\quad
J=\{b_1,\ldots,b_m\},
\quad
U=\{u_1,\ldots,u_m\},
\quad
V=\{v_1,\ldots,v_m\}.
\]
For every class label $b\in\{0,1\}$ and index pair $(j,k)\in[m]^2$, we define a distribution
$P_{b,j,k}$. The tester must output $b$; it need not output $j$ or $k$. Write $B_{jk}:=\{u_j,v_k\}$. The construction is chosen to
produce the following three conditional laws.

\begin{center}
\small
\begin{tabular}{ccl}
\hline
Query & Conditional behavior & What it determines \\
\hline
$I$ & $P_{b,j,k}(a_j\mid I)=3/4$ for both labels & index $j$ \\
$J$ & $P_{b,j,k}(b_k\mid J)=3/4$ for both labels & index $k$ \\
$B_{jk}$ & $P_{0,j,k}(u_j\mid B_{jk})=3/4$,
$P_{1,j,k}(u_j\mid B_{jk})=1/4$ & label $b$ \\
\hline
\end{tabular}
\end{center}

The following masses realize these laws. On $I$, set
\[
P_{b,j,k}(a_j)=\frac{3}{16},
\qquad
P_{b,j,k}(a_r)=\frac{1}{16(m-1)}\quad(r\ne j),
\]
and define the masses on $J$ analogously with $b_k$ in place of $a_j$.
Each of $I$ and $J$ has total mass $1/4$, independently of the label.

For the remaining points, let $\eta=m^{-4}$ and
$Z=2m-2+4\eta$. Every point in $U\cup V$ outside $\{u_j,v_k\}$ has mass
$1/(2Z)$.
On $B_{jk}$, define
\[
\begin{aligned}
\bigl(P_{0,j,k}(u_j),P_{0,j,k}(v_k)\bigr)
&=
\left(\frac{3\eta}{2Z},\frac{\eta}{2Z}\right),
\qquad
\bigl(P_{1,j,k}(u_j),P_{1,j,k}(v_k)\bigr)
&=
\left(\frac{\eta}{2Z},\frac{3\eta}{2Z}\right).
\end{aligned}
\]
All masses are positive and sum to one. The two distributions with the same
$(j,k)$ differ only at $u_j$ and $v_k$. The label signal is hidden in one very
light pair $B_{jk}$. Conditioning
exactly on $B_{jk}$ cancels its small total mass and reveals the constant
$3/4$-versus-$1/4$ bias. If a query contains any additional ordinary atom,
that atom dilutes the label-dependent discrepancy to $O(\eta)$. We choose
$\eta=m^{-4}$ so that even $O(m^2)$ off-target queries contribute only
$o(1)$ total KL divergence; therefore a non-adaptive tester essentially has
to guess which of the $m^2$ target pairs to query.
Define
\[
\Pcal_b^{(m)}:=\{P_{b,j,k}:j,k\in[m]\}.
\]

\begin{theorem}[Matching quadratic lower bound]
\label{thm:quadratic-gap}
For every $m\ge2$ and $\eps\in(0,1/2)$, write
$n_\eps:=\lceil8\log(3/\eps)\rceil$. There is a constant $c_\eps>0$ depending
only on $\eps$ such that
\[
T_{\mathrm{ad}}^\star(\eps;\Pcal_0^{(m)},\Pcal_1^{(m)})
\le 3n_\eps,
\qquad
c_\eps m^2
\le T_{\mathrm{na}}^\star(\eps;\Pcal_0^{(m)},\Pcal_1^{(m)})
\le (m^2+2)n_\eps.
\tag{4}
\label{eq:hard-family-bounds}
\]
Hence this family has adaptivity gap
$\Theta_{\eps}(m^2)=\Theta_{\eps}(N^2)$. After padding the construction with
at most three common atoms, $\Gap_N(\eps)=\Omega_{\eps}(N^2)$ for every
$N\ge8$.
\end{theorem}

\paragraph{Step 1: locate the target pair adaptively.}
Query $I$ independently $n_\eps$ times and let $\widehat j$ be the unique index
whose atom appears more than $n_\eps/2$ times, using an arbitrary fallback if no
such index exists. Since $a_j$ appears with probability $3/4$,
Hoeffding's inequality gives
\[
\Pr(\widehat j\ne j)\le e^{-n_\eps/8}.
\]
The same procedure on $J$ produces $\widehat k$ with the same error bound.
After observing these samples, query $B_{\widehat j\widehat k}$ another
$n_\eps$ times and output zero if $u_{\widehat j}$ appears more than
$n_\eps/2$ times.
Conditional on $\widehat j=j$ and $\widehat k=k$, this final decision has
error at most $e^{-n_\eps/8}$. Since this quantity is at most $\eps/3$, a
union bound proves the adaptive bound in equation~\eqref{eq:hard-family-bounds}.

\paragraph{Step 2: only the target pair has a constant separation.}
Fix $s=(j,k)$, write $P_{b,s}:=P_{b,j,k}$, and abbreviate
$P_0=P_{0,s}$, $P_1=P_{1,s}$, and
$B_s=B_{jk}$. On the target pair,
\[
\operatorname{KL}\bigl((P_0)_{B_s}\mathbin\| (P_1)_{B_s}\bigr)
=
\operatorname{KL}\bigl((3/4,1/4)\mathbin\|(1/4,3/4)\bigr)
=
\frac12\log3
=:
d_0.
\]
For $A\ne B_s$, the two conditional laws either coincide or $A$ contains a
common atom whose mass dominates the total discrepancy on $u_j$ and $v_k$.
The normalization calculation in Appendix~\ref{app:lower-bound-proof}, along
with the fact that every conditional likelihood ratio lies between $1/9$ and
$9$, gives
\[
\TV\bigl((P_0)_A,(P_1)_A\bigr)\le16\eta,
\qquad
\operatorname{KL}\bigl((P_0)_A\mathbin\|(P_1)_A\bigr)
\le32(\log9)\eta=:C_0\eta
\quad(A\ne B_s).
\tag{5}
\label{eq:body-off-target-kl}
\]
Thus a query to $B_s$ has a constant KL divergence, whereas every other
query has KL divergence at most $O(m^{-4})$.

\paragraph{Step 3: some target pair receives at most $T/m^2$ queries.}
Consider any randomized non-adaptive tester using $T$ queries. For each
$s=(j,k)$, let $N_s$ count how many submitted queries are exactly $B_s$.
For every realized query vector, $\sum_{s\in[m]^2}N_s\le T$. The distribution
of this query vector is the same for every unknown
$P_{b,j,k}$. Averaging over the $m^2$ indices therefore gives one $s$ such
that $\mathbb E N_s\le T/m^2$. This is the point at which nonadaptivity is
used: after observing responses from $I$ and $J$, an adaptive tester may
choose a different target pair for each $(j,k)$.

\paragraph{Step 4: worst-case accuracy requires $T=\Omega(m^2)$.}
For the selected $s$, the KL chain rule and
equation~\eqref{eq:body-off-target-kl} give
\[
\operatorname{KL}\bigl(\mathsf L_{P_{0,s}}(H_T)\mathbin\|
          \mathsf L_{P_{1,s}}(H_T)\bigr)
\le
d_0\mathbb E N_s+C_0\eta T
\le
T\left(\frac{d_0}{m^2}+\frac{C_0}{m^4}\right).
\]
If the tester has worst-case error at most $\eps$, then its error on both
$P_{0,s}$ and $P_{1,s}$ is at most $\eps$. Their transcript laws must
therefore have total variation distance at least $1-2\eps$. Pinsker's
inequality \citep{ali1966,reid2011,sason2016} yields
\[
T
\ge
\frac{2(1-2\eps)^2}
{\dfrac{\log3}{2m^2}+\dfrac{32\log9}{m^4}}
\ge
c_\eps m^2.
\tag{6}
\label{eq:quadratic-lower}
\]
For the matching upper bound, query $I$, $J$, and every $B_{rs}$ exactly
$n_\eps$ times, then apply the same majority rules to the stored samples. This
uses $(m^2+2)n_\eps$ queries and completes
equation~\eqref{eq:hard-family-bounds}. Appendix~\ref{app:lower-bound-proof}
verifies the off-target bound, a valid value of $c_\eps$, and padding to every
$N\ge8$.

\section{Scope and Limitations}

We study exact conditional access to a fixed response distribution. Queries restrict the range of possible responses without changing their relative probabilities within that range, isolating the role of adaptive query selection. In real interactions, testing strategies and model responses may be coupled, and history may change the model's internal state. Our guarantees do not directly cover these effects. Weakly coupled interfaces may more closely approximate this idealized model, but whether such an approximation preserves our conclusions requires further analysis.

Some specific interfaces realize this mechanism. At fixed observations and policy outputs, ideal output-layer hard masking retains the logits of allowed actions and renormalizes, implementing conditional sampling from the original action distribution. This mechanism is studied by \citet{huang2022masking} and implemented numerically in MaskablePPO.\footnote{\url{https://sb3-contrib.readthedocs.io/en/master/modules/ppo_mask.html}} Our access model additionally requires that evaluators can specify the needed masks. Task-relevant behavioral groups that preserve the distinctions being tested and have positive probability throughout the relevant hypothesis classes can control the number of categories $N$. Conditional access must then hold at the category level.

We measure cost in conditional queries. Obtaining a conditional sample by filtering ordinary samples costs $1/P(A)$ original draws on average, whereas direct masking has a different cost structure. Query counts therefore cannot directly substitute for total-cost comparisons in a concrete system.

Furthermore, the quadratic gap is a worst-case result over general hypothesis classes and does not determine the gap for typical evaluation tasks. Incorporating model-state changes, approximate conditional access, actual access costs, and task structure requires further modeling and theoretical research.

\subsection*{AI use statement}

The authors conceived the problem and developed the theory. Generative AI
assisted with proof development and checking, literature search, and editing.
The authors verified all AI-assisted content and take full responsibility for
the paper.

\subsection*{Reproducibility statement}

Sections~\ref{sec:setup}--\ref{sec:learnability} state the model and
assumptions; the appendix gives complete proofs of all results.

\clearpage
\raggedbottom
\bibliography{references}
\bibliographystyle{iclr2027_conference}

\clearpage
\flushbottom
\appendix

\section*{Clarifying the Contribution}

This paper gives precise characterizations of learnability and the value of
interaction under conditional queries, with distinct roles for the individual
results and the established mathematical tools.
Theorem~\ref{thm:learnability} uses the pairwise conditional representation,
concentration inequalities, and transcript coupling to characterize uniform
learnability of arbitrary pairs of distribution classes by positive separation
in their pairwise conditional probabilities. It also shows that zero separation
forces the optimal worst-case error to equal $1/2$ at every finite query budget.
Theorem~\ref{thm:upper-bound} builds on the pairwise CFTP exact-sampling mechanism
of \citet{fotakis2022}. The coalescence structure under uniform sampling over all
pairs, together with exponential-moment control of total sample consumption,
allows pair samples collected in advance to simulate the entire transcript of
any adaptive policy within total variation error $\rho$, using
$O(N^2(T+\log(1/\rho)))$ queries independently of the minimum atom probability.
Theorem~\ref{thm:quadratic-gap} constructs a testing problem in which an adaptive
tester first locates the relevant indices and then selects the target event.
Conditional KL divergence and an averaging argument give a matching lower bound
against non-adaptive testers allowed to query arbitrary events.
Together, these results show that, in the exact conditional-query model for a
fixed distribution, interaction does not change which pairs of hypothesis classes
can be uniformly distinguished, but can change the number of queries required. At a
fixed error level, the worst-case efficiency gain is exactly
$\Theta_{\eps}(N^2)$ in the explicit outcome-space size $N$.

\section*{Additional Related Work}

Dynamic benchmark projects create or refresh data using model behavior over successive development cycles \citep{kiela2021,nie2020,potts2021}. Behavioral testing broadens what benchmark items probe, while interactive platforms evaluate open-ended or multi-turn behavior \citep{ribeiro2020,thrush2022,zheng2023}. Live benchmarks update questions over time to limit contamination \citep{white2025}. These approaches change evaluation design at different timescales.

Adaptive question selection has a long history in psychometrics and computerized testing \citep{weiss1982,gershon2005,ghosh2021}. Item-response models have also been used to analyze the difficulty and discriminative value of NLP evaluation examples \citep{lalor2016,rodriguez2021,maiapolo2024}. Recent model-evaluation work uses predicted item difficulty or current model performance to choose questions more efficiently \citep{truong2025,zhuang2025,ding2026}, and active estimators reduce the number of benchmark items needed for score estimation \citep{wu2026}.

The timing of experimental choices also appears in active hypothesis testing. This line begins with sequential experiment design and studies how actions chosen from past observations affect testing performance \citep{chernoff1959,naghshvar2013active,nitinawarat2013}. In particular, \citet{naghshvar2013} compare adaptive and non-adaptive action selection for a specified finite collection of hypotheses, actions, and observation laws.

\section{Proof of the Learnability Criterion}
\label{app:learnability-proof}

We first prove the stability relation used in Section~\ref{sec:learnability}.

\begin{lemma}[Pairwise stability of conditional queries]
\label{lem:event-stability}
For any $P,Q\in\Delta_+(\mathcal X)$ and any nonempty
$A\subseteq\mathcal X$,
\[
\TV(P_A,Q_A)
\le
(|A|-1)d_{\mathrm{pair}}(P,Q).
\]
Consequently, equation~\eqref{eq:event-stability} holds.
\end{lemma}

\begin{proof}
The claim is immediate when $|A|=1$. Suppose $|A|=m\ge2$ and write
$p=P_A$, $q=Q_A$, and $d=d_{\mathrm{pair}}(P,Q)$. For every distinct
$i,j\in A$, conditioning $p$ and $q$ again on $\{i,j\}$ gives the same
pairwise conditionals as conditioning $P$ and $Q$ directly. Hence
\[
\frac{|p_iq_j-q_ip_j|}{(p_i+p_j)(q_i+q_j)}
\le d.
\tag{A.1}
\label{eq:pair-determinant}
\]
Let $I_+:=\{i\in A:p_i\ge q_i\}$ and $I_-:=A\setminus I_+$. Then
\begin{align*}
\TV(p,q)
&=
\sum_{i\in I_+}\sum_{j\in I_-}(p_iq_j-q_ip_j)\\
&\le
d\sum_{i\in I_+}\sum_{j\in I_-}(p_i+p_j)(q_i+q_j)\\
&\le
d\sum_{\{i,j\}\subseteq A}(p_i+p_j)(q_i+q_j).
\end{align*}
The final sum equals
\[
(m-1)\sum_i p_iq_i+\sum_{i\ne j}p_iq_j
=
1+(m-2)\sum_i p_iq_i
\le
m-1.
\]
This proves the lemma.
\end{proof}

\begin{lemma}[Adaptive transcript stability]
\label{lem:transcript-stability}
Fix any randomized adaptive tester using $T$ queries. If
$\mathsf L_P(H_T)$ and $\mathsf L_Q(H_T)$ denote its transcript laws under
$P$ and $Q$, then
\[
\TV\bigl(\mathsf L_P(H_T),\mathsf L_Q(H_T)\bigr)
\le
T(N-1)d_{\mathrm{pair}}(P,Q).
\]
\end{lemma}

\begin{proof}
Couple the two executions one round at a time. As long as their histories
agree, use the same private randomness to select the same event $A_t$, then
couple the two observations optimally. Lemma~\ref{lem:event-stability} bounds
the conditional probability that the observations first differ at round $t$
by $(N-1)d_{\mathrm{pair}}(P,Q)$. A union bound over the $T$ rounds bounds the
probability that the transcripts differ. The coupling characterization of
total variation gives the result.
\end{proof}

\begin{proof}[Proof of Theorem~\ref{thm:learnability}]
First suppose $\Delta_{\mathrm{pair}}(\Pcal_0,\Pcal_1)=0$. The case $T=0$
follows directly from the absence of observations. For $T\ge1$ and every
$\eta>0$, choose $P\in\Pcal_0$ and $Q\in\Pcal_1$ such that
\[
d_{\mathrm{pair}}(P,Q)
<
\frac{\eta}{T(N-1)}.
\]
Lemma~\ref{lem:transcript-stability} makes the two transcript laws closer than
$\eta$ in total variation. For any decision rule, the sum of its error
probabilities under $P$ and $Q$ is at least $1-\eta$. Its worst-case error is
therefore at least $(1-\eta)/2$. Letting $\eta$ tend to zero gives an adaptive
minimax risk of at least $1/2$. A fair random guess achieves $1/2$, so the
adaptive risk is exactly $1/2$. The same conclusion follows for non-adaptive
testing from
$R_T^{\mathrm{ad}}\le R_T^{\mathrm{na}}\le1/2$.

Now suppose $\Delta:=\Delta_{\mathrm{pair}}(\Pcal_0,\Pcal_1)>0$ and fix
$\delta\in(0,1/2)$. Query every pair
\[
n
:=
\left\lceil
\frac{8}{\Delta^2}\log\frac{2M}{\delta}
\right\rceil
\]
times. Let $\widehat\theta_{ij}$ be the empirical frequency of outcome $i$
from queries to $\{i,j\}$, and let
$\widehat\Theta=(\widehat\theta_{ij})_{i<j}$. Hoeffding's inequality and a
union bound give
\[
\Pr_P\left(
\|\widehat\Theta-\Theta(P)\|_\infty\ge\frac{\Delta}{4}
\right)
\le
2M\exp\left(-\frac{n\Delta^2}{8}\right)
\le
\delta.
\]
For $b\in\{0,1\}$, define
\[
D_b
:=
\inf_{Q\in\Pcal_b}
\|\widehat\Theta-\Theta(Q)\|_\infty,
\]
and output a label minimizing $D_b$. If $P\in\Pcal_b$ and
$\|\widehat\Theta-\Theta(P)\|_\infty<\Delta/4$, then $D_b<\Delta/4$ while
$D_{1-b}\ge3\Delta/4$. The decision is therefore correct with probability at
least $1-\delta$, uniformly over both classes. The test is non-adaptive and
uses $Mn$ queries. This proves the finite-sample bound and all three
equivalences.
\end{proof}

\section{Proof of the Universal Transcript Simulation}
\label{app:upper-bound-proofs}

\paragraph{Relation to prior perfect-sampling results.}
For $A=\mathcal X$, our endpoint-merging update and backward composition coincide
with Algorithm~1 of \citet{fotakis2022} when its pair distribution is uniform over
all pairs. The single-sample exactness mechanism therefore follows existing CFTP
constructions. Their Theorem~2 bounds the sampling cost through the minimum atom
probability and the spectral gap of the induced Markov chain; their Theorem~3
obtains a distribution-structure-independent guarantee under an additional
bounded-probability-ratio assumption. Our analysis instead uses the complete,
uniformly sampled pair family to obtain a cost bound valid for every full-support
distribution. In the forward image process, when $k$ states remain, the probability
of a merger is exactly $\binom{k}{2}/\binom{N}{2}$, regardless of the distribution.
Lemma~\ref{lem:cftp-cost} consequently gives an exact geometric-waiting-time
decomposition and a uniform exponential-moment bound for every conditioning
event $A$. To obtain Theorem~\ref{thm:upper-bound}, we then simulate events chosen
from the preceding simulated history, using the independence of the unconsumed
sample stream, and control the total consumption across all rounds. This yields
the fixed budget $O(N^2(T+\log(1/\rho)))$ for the entire transcript, without a
minimum-mass or probability-ratio assumption and without an additional $\log T$
factor. The matching lower bound in Theorem~\ref{thm:quadratic-gap} further shows
that the quadratic overhead is unavoidable for conditional hypothesis testing,
even when non-adaptive testers may query arbitrary events.

Fix $P\in\Delta_+(\mathcal X)$. Let
$\omega_s=(E_s,Y_s)$, $s\ge1$, be an i.i.d.\ source with
\[
E_s\sim\operatorname{Unif}\left(\binom{\mathcal X}{2}\right),
\qquad
Y_s\mid E_s=e\sim P_e.
\]
For a nonempty event $A$, let $F_s^A:=\Phi_{E_s,Y_s}^A$ be the random map
defined in Section~\ref{sec:upper-bound}.

\begin{lemma}[Stationarity of the pairwise maps]
\label{lem:map-stationarity}
If $X\sim P_A$ is independent of $F_s^A$, then
$F_s^A(X)\sim P_A$.
\end{lemma}

\begin{proof}
Condition on $E_s=e$. If $e\not\subseteq A$, the map is the identity. If
$e\subseteq A$, it fixes every state outside $e$ and replaces a state in $e$
by an independent draw from $P_e$. For $z\in e$, the resulting mass at $z$
is
\[
P_A(e)P_e(z)
=
\frac{P(e)}{P(A)}\frac{P(z)}{P(e)}
=
P_A(z).
\]
The mass of each $z\in A\setminus e$ is unchanged. Thus every conditional
map kernel preserves $P_A$, and averaging over $E_s$ proves the claim.
\end{proof}

Interpret $F_1^A$ as the map from time $-1$ to time $0$, $F_2^A$ as the map
from time $-2$ to time $-1$, and so on. Define
\[
C_n^A
:=
F_1^A\circ F_2^A\circ\cdots\circ F_n^A,
\qquad
\tau_A
:=
\inf\{n\ge0:C_n^A\text{ is constant on }A\}.
\]
When $\tau_A<\infty$, write $W_A$ for the common value of
$C_{\tau_A}^A$.

\begin{lemma}[Exactness and coalescence cost]
\label{lem:cftp-cost}
For every nonempty $A\subseteq\mathcal X$, $\tau_A<\infty$ almost surely and
$W_A\sim P_A$. If $a:=|A|\ge2$, then
\[
\tau_A
\overset{d}{=}
\sum_{k=2}^{a}G_k,
\qquad
G_k\sim\operatorname{Geom}\left(\frac{\binom k2}{M}\right)
\text{ independently}.
\tag{B.1}
\label{eq:appendix-coalescence-law}
\]
Consequently,
\[
\mathbb E\tau_A
=
2M\left(1-\frac1a\right)
<2M,
\tag{B.2}
\label{eq:mean-coalescence}
\]
and
\[
\mathbb E\exp\left(\frac{\tau_A}{8M}\right)
\le e.
\tag{B.3}
\label{eq:coalescence-mgf}
\]
For $a=1$, the same conclusions hold with $\tau_A=0$.
\end{lemma}

\begin{proof}
For each fixed $n$, reversal invariance of the i.i.d.\ maps gives
\[
F_1^A\circ\cdots\circ F_n^A
\overset{d}{=}
F_n^A\circ\cdots\circ F_1^A.
\]
Consider the forward image process $S_0=A$ and
$S_r=F_r^A(S_{r-1})$. If $|S_{r-1}|=k$, its size decreases by one exactly
when both endpoints of $E_r$ lie in $S_{r-1}$. This event has probability
\[
p_k:=\frac{\binom k2}{M}.
\]
If the pair contains at most one state in the current image, the map either
fixes that image or replaces one state by another and its size stays $k$.
Thus the forward time to reach a singleton is the sum of independent
geometric waiting times in equation~\eqref{eq:appendix-coalescence-law}.
The reversal identity shows that this hitting time has the same distribution
as $\tau_A$. It is finite almost surely.

Lemma~\ref{lem:map-stationarity} also proves exactness by the standard
coupling-from-the-past argument \citep{propp1996}. For completeness, let
$Z_n\sim P_A$ be independent of the first $n$ maps. Applying the maps from
the oldest to the most recent gives
$C_n^A(Z_n)\sim P_A$. On $\{\tau_A\le n\}$, this value equals $W_A$ for
every $Z_n$. Hence
\[
\TV(\mathsf L(W_A),P_A)
\le
\Pr(\tau_A>n),
\]
and the right-hand side tends to zero.

The mean follows from
\[
\sum_{k=2}^{a}\frac{1}{\binom k2}
=
2\left(1-\frac1a\right).
\]
It remains to prove the exponential moment. Set $s=1/(8M)$. If
$G\sim\operatorname{Geom}(p)$ on $\{1,2,\ldots\}$, then
\[
\mathbb E e^{sG}
=
\frac{pe^s}{1-(1-p)e^s}
=
\frac{e^s}{1-u},
\qquad
u:=\frac{(1-p)(e^s-1)}{p}.
\]
For every $p=p_k$, we have $s\le p/8$. Since $s\le1/8$ and
$e^s-1\le2s$, it follows that $u\le1/4$. Therefore
\[
\log\mathbb E e^{sG}
=
s-\log(1-u)
\le
s+\frac43u
\le
\frac{4s}{p}.
\]
Using independence in equation~\eqref{eq:appendix-coalescence-law},
\begin{align*}
\log\mathbb E e^{s\tau_A}
&\le
4s\sum_{k=2}^{a}\frac1{p_k}\\
&=
\frac12\sum_{k=2}^{a}\frac1{\binom k2}\\
&=
1-\frac1a
\le1,
\end{align*}
which proves equation~\eqref{eq:coalescence-mgf}.
\end{proof}

\begin{proof}[Proof of Theorem~\ref{thm:upper-bound}]
First allow an infinite i.i.d.\ source $(\omega_s)_{s\ge1}$. Draw the private
randomness of $\pi$ independently. At round $t$, the policy chooses $A_t$
from the simulated history. Starting at the next unused source element, run
the construction above until its maps coalesce on $A_t$, and set
$\widetilde X_t=W_{A_t}$.

The number of elements read is a stopping time for the fresh source segment.
The strong Markov property of an i.i.d.\ sequence implies that the unused tail
is again i.i.d.\ and independent of the maps already consumed. This is the
read-once coupling-from-the-past construction of \citet{wilson2000}. By
Lemma~\ref{lem:cftp-cost}, conditionally on the simulated history,
$\widetilde X_t\sim P_{A_t}$. Induction over $t$ therefore shows that the
infinite-stream simulated transcript has exactly the same law as $H_T^\pi$.

Let $\tau_t$ be the number of source elements consumed at round $t$ and let
$S_T=\sum_{t=1}^T\tau_t$. Equation~\eqref{eq:coalescence-mgf} holds
conditionally on every past simulated history, uniformly over the selected
event. Iterating conditional expectations gives
\[
\mathbb E\exp\left(\frac{S_T}{8M}\right)
\le
e^T.
\]
Markov's inequality now yields
\begin{align*}
\Pr\left[
S_T>8M\left(T+\log\frac1\rho\right)
\right]
&\le
\exp\left(-T-\log\frac1\rho\right)e^T\\
&=
\rho.
\tag{B.4}
\label{eq:stream-overflow}
\end{align*}

The finite simulation chooses the $L$ pairs in
equation~\eqref{eq:simulation-budget} before observing any response and runs the
same construction until either all $T$ rounds are simulated or the source
is exhausted. If exhaustion occurs in round $t$, retain all completed rounds and set the current response to $\widetilde X_t=\min A_t$. In each subsequent round $r$, select $A_r$ according to $\pi_r$ given the completed history and set $\widetilde X_r=\min A_r$.
Couple it to the infinite construction using the same first $L$ source
elements. Their transcripts agree unless $S_T>L$, an event of probability at
most $\rho$ by equation~\eqref{eq:stream-overflow}. The coupling
characterization of total variation proves the theorem. Applying the same
decision rule changes its error by at most $\rho$.
\end{proof}

\begin{proof}[Proof of Corollary~\ref{cor:quadratic-upper}]
Suppose a $T$-query adaptive tester has worst-case error at most
$\eps<1/2$. For any $\delta\in(0,1/2)$, let $K_{\eps,\delta}$ be the smallest
odd integer at least
\[
\frac{2}{(1-2\eps)^2}\log\frac{2}{\delta}.
\]
Run $K=K_{\eps,\delta}$ independent copies and take a majority vote.
Hoeffding's inequality gives error at most
\[
\exp\left(-\frac{K(1-2\eps)^2}{2}\right)
\le
\frac{\delta}{2}.
\]
This amplified tester uses $KT$ adaptive queries. Apply
Theorem~\ref{thm:upper-bound} with $\rho=\delta/2$. The simulated tester has
error at most $\delta$ and uses at most
\[
\left\lceil
8M\left(K_{\eps,\delta}T+\log\frac{2}{\delta}\right)
\right\rceil
\tag{B.5}
\label{eq:testing-upper-bound}
\]
pair queries.

For the fixed-error gap, set $T=T_{\mathrm{ad}}^\star(\eps;\Pcal_0,\Pcal_1)$
and $\delta=\eps$. If the minimax infimum at horizon $T$ is not attained,
choose a $T$-query adaptive tester with error below any fixed
$\bar\eps\in(\eps,1/2)$ and apply the same argument with
$K_{\bar\eps,\eps}$. Since the definition of $\Gap_N(\eps)$ restricts to
$T\ge1$, division by $T$ gives a constant depending only on $\eps$ times
$M=\binom N2$.
\end{proof}

\section{Proof of the Quadratic Lower Bound}
\label{app:lower-bound-proof}

\begin{proof}[Proof of Theorem~\ref{thm:quadratic-gap}]
Fix $m\ge2$ and $\eps\in(0,1/2)$, and write
$n_\eps=\lceil8\log(3/\eps)\rceil$. The masses in the construction are
positive. Moreover,
\[
P_{b,j,k}(I)=P_{b,j,k}(J)=\frac14,
\qquad
P_{b,j,k}(U\cup V)
=
\frac{2m-2+4\eta}{2Z}
=
\frac12,
\]
so each $P_{b,j,k}$ is a full-support probability distribution.

Conditional on $I$, the distinguished atom $a_j$ has probability $3/4$.
Let $\widehat j$ be the unique index whose atom appears more than
$n_\eps/2$ times among $n_\eps$ samples from $P_{b,j,k}(\cdot\mid I)$, with
a fixed fallback when no such index exists. Hoeffding's inequality gives
\[
\Pr(\widehat j\ne j)
\le
\Pr\bigl(\operatorname{Bin}(n_\eps,3/4)\le n_\eps/2\bigr)
\le
e^{-n_\eps/8}.
\]
The same procedure using $J$ gives
$\Pr(\widehat k\ne k)\le e^{-n_\eps/8}$. Conditional on
$\widehat j=j$ and $\widehat k=k$, thresholding the frequency of $u_j$ in
$n_\eps$ samples from $B_{jk}$ at $1/2$ has error at most
$e^{-n_\eps/8}$ under either label. Since
$3e^{-n_\eps/8}\le\eps$, the adaptive tester uses $3n_\eps$ queries and has
worst-case error at most $\eps$. A non-adaptive tester obtains the same error
by querying $I$ and $J$ $n_\eps$ times each and every pair
$B_{rs}=\{u_r,v_s\}$ $n_\eps$ times before observing any response. It then
uses the samples from $B_{\widehat j\widehat k}$. This proves both upper bounds
in equation~\eqref{eq:hard-family-bounds}.

It remains to prove the non-adaptive lower bound. Fix
$s=(j,k)$, write $P_{b,s}:=P_{b,j,k}$, and abbreviate
$P_0=P_{0,s}$, $P_1=P_{1,s}$, and
$B_s=B_{jk}$. The target pair satisfies
\[
d_0
:=
\operatorname{KL}\bigl((P_0)_{B_s}\mathbin\|(P_1)_{B_s}\bigr)
=
\operatorname{KL}\bigl((3/4,1/4)\mathbin\|(1/4,3/4)\bigr)
=
\frac12\log3.
\]
We next bound every other query. For a nonempty event $A$, let
$\mu_A=(P_0)_A$ and $\nu_A=(P_1)_A$. The unconditional likelihood ratio
$P_0(x)/P_1(x)$ belongs to $\{1/3,1,3\}$, which implies
\[
\frac19
\le
\frac{\mu_A(x)}{\nu_A(x)}
\le
9
\qquad(x\in A).
\]
For positive finite measures $\alpha$ and $\beta$ on $A$, set
$a=\alpha(A)$ and $b=\beta(A)$. Then
\begin{align*}
2\TV(\alpha_A,\beta_A)
&=
\sum_{x\in A}\left|\frac{\alpha(x)}a-\frac{\beta(x)}b\right|\\
&\le
\frac{\sum_{x\in A}|\alpha(x)-\beta(x)|}{a}
+
\frac{|a-b|}{a}.
\end{align*}
Since $|a-b|\le\sum_{x\in A}|\alpha(x)-\beta(x)|$, and the same inequality
holds after exchanging $a$ and $b$,
\[
\TV(\alpha_A,\beta_A)
\le
\frac{\sum_{x\in A}|\alpha(x)-\beta(x)|}
{\min\{\alpha(A),\beta(A)\}}.
\]

The distributions $P_0$ and $P_1$ differ only at $u_j$ and $v_k$, with
unconditional $\ell_1$ discrepancy $2\eta/Z$. If $A$ contains neither active
atom, or if it is a singleton active atom, then $\mu_A=\nu_A$. In every other
case with $A\ne B_s$, the event $A$ contains a common atom. Every common atom
has mass at least $1/[16(m-1)]$ under both labels. For inactive atoms in
$U\cup V$, this follows from $Z\le8(m-1)$. Since $Z\ge2(m-1)$, the preceding
normalization inequality gives
\[
\TV(\mu_A,\nu_A)
\le
\frac{2\eta/Z}{1/[16(m-1)]}
\le
16\eta
\qquad(A\ne B_s).
\]
If two distributions $\mu$ and $\nu$ have likelihood ratios in $[1/R,R]$,
then
\begin{align*}
\operatorname{KL}(\mu\|\nu)+\operatorname{KL}(\nu\|\mu)
&=
\sum_x(\mu(x)-\nu(x))\log\frac{\mu(x)}{\nu(x)}\\
&\le
2(\log R)\TV(\mu,\nu).
\end{align*}
Consequently,
\[
\operatorname{KL}(\mu_A\|\nu_A)
\le
32(\log9)\eta
=:
C_0\eta
\qquad(A\ne B_s).
\]

Now fix an arbitrary randomized non-adaptive tester using $T$ queries. Its
query vector $A_{1:T}$ has the same law under every distribution in the two
classes. For $s\in[m]^2$, let
$N_s=\sum_{t=1}^T\mathbf1\{A_t=B_s\}$. Since
$\sum_{s\in[m]^2}N_s\le T$ for every realized query vector, there is an
$s\in[m]^2$ for which $\mathbb E N_s\le T/m^2$. Conditional on the query
vector, the responses are independent. The KL chain rule therefore gives
\begin{align*}
&\operatorname{KL}\bigl(
\mathsf L_{P_{0,s}}(H_T)\mathbin\|\mathsf L_{P_{1,s}}(H_T)
\bigr)\\
&\qquad=
\mathbb E\left[
\sum_{t=1}^T
\operatorname{KL}\bigl((P_{0,s})_{A_t}\mathbin\|(P_{1,s})_{A_t}\bigr)
\right]\\
&\qquad\le
d_0\mathbb E N_s+C_0\eta T
\le
T\left(\frac{d_0}{m^2}+\frac{C_0}{m^4}\right).
\end{align*}
If the tester has worst-case error at most $\eps$, data processing through its
decision rule implies that these transcript laws have total variation at least
$1-2\eps$. Pinsker's inequality yields
\[
2(1-2\eps)^2
\le
T\left(\frac{d_0}{m^2}+\frac{C_0}{m^4}\right).
\]
For $m\ge2$,
\[
\frac{d_0}{m^2}+\frac{C_0}{m^4}
\le
\frac{\frac12\log3+8\log9}{m^2}.
\]
Thus equation~\eqref{eq:quadratic-lower} holds with
\[
c_\eps
:=
\frac{2(1-2\eps)^2}{\frac12\log3+8\log9}.
\]
Since a zero-query test has worst-case error at least $1/2$, the adaptive query
complexity of this family is at least one. Combining this fact with the three
bounds in equation~\eqref{eq:hard-family-bounds} shows that its adaptivity gap
is $\Theta_\eps(m^2)=\Theta_\eps(N^2)$.

Finally, let $N\ge8$, set $m=\lfloor N/4\rfloor$, and write $r=N-4m$. The
case $r=0$ is already covered. If $r>0$, add atoms $c_1,\ldots,c_r$ and define
\[
\widetilde P_{b,j,k}(x)
=
\frac78P_{b,j,k}(x)
\quad(x\in I\cup J\cup U\cup V),
\qquad
\widetilde P_{b,j,k}(c_\ell)
=
\frac{1}{8r}.
\]
Let
$\widetilde\Pcal_b^{(N)}=\{\widetilde P_{b,j,k}:j,k\in[m]\}$. The masses sum
to $7/8+1/8=1$, so these are full-support distributions on an outcome space
of size $N$. Conditional laws on events
contained in the original four blocks are unchanged, so the adaptive upper
bound and the target-pair divergence remain unchanged.

Every off-target event that distinguishes the labels contains a common atom.
Each such atom has mass at least $7/(128m)$. Indeed, this follows from the
original masses for atoms in $I\cup J$, from $Z\le8m$ for inactive atoms in
$U\cup V$, and from $r\le3$ and $1/(8r)\ge7/(128m)$ for the new atoms. The
unconditional $\ell_1$ discrepancy
is
\[
\frac78\frac{2\eta}{Z}
=
\frac{7\eta}{4Z}
\le
\frac{7\eta}{4m},
\]
where $Z\ge m$. Hence every off-target conditional total variation is at most
$32\eta$. The conditional likelihood ratios remain in $[1/9,9]$, so every
off-target KL divergence is at most $64(\log9)\eta$. For any randomized
non-adaptive tester with $T$ queries, let $\widetilde N_s$ count the queries
equal to $B_s$. Since $\sum_s\widetilde N_s\le T$, there is an $s\in[m]^2$
such that $\mathbb E\widetilde N_s\le T/m^2$. The KL chain rule then gives
\[
\operatorname{KL}\bigl(
\mathsf L_{\widetilde P_{0,s}}(H_T)
\mathbin\|
\mathsf L_{\widetilde P_{1,s}}(H_T)
\bigr)
\le
T\left(
\frac{d_0}{m^2}
+
\frac{64\log9}{m^4}
\right).
\]
If the tester has worst-case error at most $\eps$, Pinsker's inequality implies
\[
2(1-2\eps)^2
\le
T\left(
\frac{d_0}{m^2}
+
\frac{64\log9}{m^4}
\right),
\]
and therefore $T=\Omega_\eps(m^2)$. The adaptive tester above still uses at
most $3n_\eps$ queries. Since $m=\lfloor N/4\rfloor\ge N/8$, the padded family
has adaptivity gap $\Omega_\eps(N^2)$. This proves the theorem.
\end{proof}

\end{document}